\documentclass[10pt,journal]{IEEEtran}
\usepackage{graphicx}
\usepackage{amsmath,amssymb,bm} 
\usepackage{amsthm}
\usepackage{algorithmic}
\usepackage{amsfonts}
\usepackage{color}
\usepackage{url}
\usepackage{array}
\usepackage{algorithm}
\usepackage{xspace}
\usepackage{subfigure}
\usepackage[numbers]{natbib}
\usepackage{booktabs}
\usepackage{enumerate}
\usepackage{pgfplots}
\pgfplotsset{compat=newest}
\usepgfplotslibrary{external}
\usepackage{tikz}
\usetikzlibrary{shapes.misc}
\pgfdeclarelayer{nodelayer}
\pgfdeclarelayer{edgelayer}
\pgfsetlayers{edgelayer,nodelayer,main}

\newcommand{\md}{\bm{D}}
\newcommand{\mm}{\bm{M}} 
\newcommand{\mx}{\bm{X}}       
\newcommand{\my}{\bm{Y}}       
\newcommand{\mi}{\bm{I}}       
       
\newcommand{\mb}{\bm{B}}       
\newcommand{\ms}{\bm{S}}
\newcommand{\ma}{\bm{A}}

\newcommand{\vx}{\bm{x}}
\newcommand{\vy}{\bm{y}}

\newcommand{\vu}{\bm{u}}

\newcommand{\Dc}{\mathcal{D}}

\newcommand{\Sc}{\mathcal{S}}

\newcommand{\norm}[1]{\|{#1}\|}

\newcommand{\reals}{\mathbb{R}}

\newcommand{\ld}{\text{ld}}
\newcommand{\algo}{\textsc{Gmml}\xspace}

\DeclareMathOperator{\trace}{tr}

\newtheorem{theorem}{Theorem}

\newtheorem{proposition}[theorem]{Proposition}

\theoremstyle{definition}
\newtheorem{definition}[theorem]{Definition}

\theoremstyle{remark}

\begin{document}

\tikzset{mark options={mark size=5, line width=3pt}}
%
\title{Geometric Mean Metric Learning}
\author{Pourya Habib Zadeh\quad Reshad Hosseini\quad Suvrit Sra\IEEEcompsocitemizethanks{\IEEEcompsocthanksitem Pourya Habib Zadeh and Reshad Hosseini are with the School of ECE, College of Engineering, University of Tehran.\protect\\
E-mail: \{p.habibzadeh, reshad.hosseini\}@ut.ac.ir}
\IEEEcompsocitemizethanks{\IEEEcompsocthanksitem Suvrit Sra is with the Massachusetts Institute of Technology.\protect\\
E-mail: suvrit@mit.edu}
}


\maketitle

\begin{abstract}
  We revisit the task of learning a Euclidean metric from data. We approach this problem from first principles and formulate it as a surprisingly simple optimization problem. Indeed, our formulation even admits a closed form solution. This solution possesses several very attractive properties: (i) an innate geometric appeal through the Riemannian geometry of positive definite matrices; (ii) ease of interpretability; and (iii) computational speed several orders of magnitude faster than the widely used LMNN and ITML methods. Furthermore, on standard benchmark datasets, our closed-form solution consistently attains higher classification accuracy.
\end{abstract}

%

\IEEEpeerreviewmaketitle


\section{Introduction}
Many machine learning algorithms require computing distances between input data points, be it for clustering, classification, or search. Selecting the distance measure is, therefore, an important concern; though the answer is task specific. When supervised or weakly supervised information is available, selection of the distance function can itself be cast as a learning problem called ``metric learning''~\cite{kulis2012metric,weinberger2009distance}.

In its most common form, metric learning seeks to learn a Euclidean metric. An abstract approach is to take input data in $\reals^n$ and learn a linear map $\Phi : \reals^n \to \reals^m$, so that the Euclidean distance $\|\Phi(\vx)-\Phi(\vy)\|$ can be used to measure the distance between points $\vx, \vy \in \reals^n$. More generally, the map $\Phi$ can also be nonlinear. 

The problem of learning linear maps was introduced in~\cite{xing2002distance} as ``Mahalanobis metric learning.'' Since then metric learning has witnessed a sequence of improvements both in modeling and algorithms (see related work). More broadly, the idea of linearly transforming input features is a bigger theme across machine learning and statistics; encompassing whitening transforms, linear dimensionality reduction, Euclidean metric learning, and more~\cite{kulis2012metric,cunningham2015linear}.

We revisit the task of learning a Euclidean metric. Like most Euclidean metric learning methods, we also seek to learn a \emph{Mahalanobis distance}\footnote{This is actually a squared distance. The true metric is $\sqrt{d_{\ma}}$; but in accord with metric learning literature we call~\eqref{MD} a distance.}
\begin{equation}
  \label{MD}
  d_{\ma}(\vx,\vx') = (\vx-\vx')^{T}\ma(\vx-\vx'),
\end{equation}
where $\vx, \vx' \in \mathbb{R}^d$ are input vectors, and $\ma$ is a $d\times d$ real, symmetric positive definite (SPD) matrix\footnote{Do not confuse SPD with positive semi-definite matrices.}. Like other metric learning approaches we also assume weak-supervision, which is provided through the sets of pairs
\begin{align*}
  \Sc &:= \{ (\vx_i,\vx_j) \mid \vx_i \; \text{and} \; \vx_j \; \text{are in the same class} \}\\
  \Dc &:= \{ (\vx_i,\vx_j) \mid \vx_i \; \text{and} \; \vx_j \; \text{are in different classes} \}.
\end{align*}
Unlike other Euclidean metric learning methods, however, we follow a much simpler yet fresh new approach. 

Specifically, we make the following main contributions:
\begin{list}{--}{\leftmargin=1em}\vspace*{6pt}
  \setlength{\itemsep}{0pt}
\item \textbf{Formulation.} We formulate Euclidean metric learning from first principles following intuitive geometric reasoning; we name our setup ``Geometric Mean Metric Learning'' (\algo) and cast it as an \emph{unconstrained} smooth, strictly convex optimization problem. \vspace*{4pt}
\item \textbf{Solution \& insights.} We show that our formulation admits a closed form solution, which not only also enjoys  connections to the Riemannian geometry of SPD matrices (and thus explains the name \algo) but also has important empirical consequences. \vspace*{4pt}
\item \textbf{Validation.} We consider multi-class classification using the learned metrics, and validate \algo by comparing it against widely used metric learning methods. \algo runs up to three orders of magnitude faster while consistently delivering equal or higher classification accuracy.
\end{list}

%

\subsection{Related work}
\label{sec.related}
We recall below some related work to help place \algo in perspective. We omit a discussion of nonlinear methods, and other variations of the basic Euclidean task outlined above; for these, we refer the reader to both kernelized metric learning~\cite{davis2007information} and other techniques as summarized in the recent surveys of~\citet{kulis2012metric} and~\citet{bellet2013survey}.

Probably the earliest work to formulate metric learning is~\cite{xing2002distance}, sometimes referred to as MMC. This method minimizes the sum of distances over similar points while trying to ensure that dissimilar points are far away from each other. Using the sets $\Sc$ and $\Dc$, MMC solves the optimization problem 
\begin{equation}
\label{eq:1}
  \begin{split}
    \min_{\ma \succeq 0} \quad & \sum_{(\vx_i , \vx_j) \in\mathcal{S}}{d_{\ma}(\vx_i,\vx_j)} \\
    \text{such that} \quad & \sum_{(\vx_i , \vx_j) \in \mathcal{D}}{\sqrt{d_{\ma}(\vx_i,\vx_j)}} \geq 1.
\end{split}
\end{equation}
\citet{xing2002distance} use $\sqrt{d_{\ma}}$ instead of the distance $d_{\ma}$  because under $d_{\ma}$, problem~\eqref{eq:1} has a trivial rank-one solution. To optimize~\eqref{eq:1}, they use a gradient-descent algorithm combined with a projection onto the set of positive semi-definite matrices.
The term $\sum_{(\vx_i , \vx_j) \in \mathcal{S}}{d_{\ma}(\vx_i,\vx_j)}$ is also used in the other metric learning methods like LMNN~\cite{weinberger2009distance} and MCML~\cite{globerson2005metric} as a part of their cost functions.


Information-Theoretic Metric Learning (ITML)~\cite{davis2007information}, aims to satisfy the similarity and dissimilarity constraints while staying as ``close'' as possible to a predefined matrix. This closeness is measured using the \emph{LogDet divergence} $D_{\ld}(\ma,\ma_0) := \trace(\ma \ma_0^{-1}) - \log\det(\ma \ma_0^{-1}) - d$; and ITML is formulated as follows:
\begin{equation}
\begin{split}
\min_{\ma \succeq 0} \quad & D_{\ld}(\ma,\ma_0) \\
\text{such that} \quad &  d_{\ma}(\vx,\vy) \leq u, \;\;\;\; (\vx , \vy) \in \mathcal{S},\\
& d_{\ma}(\vx,\vy) \geq l, \;\;\;\; (\vx , \vy) \in \mathcal{D},
\end{split}
\end{equation}
where $u, v \in \mathbb{R}$ are threshold parameters, chosen to encourage distance between similar points to be small and between dissimilar points be large. 
Similar to ITML,~\citet{meyer2011regression} propose the formulation
\begin{equation}
\label{eq:10}
\begin{split}
\min_{\ma \succeq 0} \quad & \sum_{(\vx_i , \vx_j) \in \mathcal{S}}{\max \bigl(\;0\;,\;l-d_{\ma}(\vx_i,\vx_j)\; \bigr)^2} \\
&+ \sum_{(\vx_i , \vx_j) \in \mathcal{D}}{\max \bigl(\;0\;,\;d_{\ma}(\vx_i,\vx_j)-u\; \bigr)^2},
\end{split}
\end{equation}
for which they use Riemannian techniques to minimize the cost function. Although~\eqref{eq:10} does not use any regularizer, the authors observed good classification performance.

There exist several attempts for achieving high scalability with both the dimensionality and the number of constraints in the metric learning methods; some examples include~\cite{shalev2004online,jain2009online,weinberger2008fast,shalit2012online}. 

However, the focus of our paper is different: we are concerned with the formulation of Euclidean metric learning. Remarkably, our new formulation admits a closed form solution, which turns out to be 3 orders of magnitude faster than established competing methods.


\section{\algo: formulation and solution}

As discussed above, the guiding idea behind Euclidean metric learning is to ultimately obtain a metric that yields ``small'' distances for similar points and ``big'' ones for dissimilar ones. Different metric learning methods try to fulfill this guideline either implicitly or explicitly.


The main idea that we introduce below is in how we choose to include the impact of the dissimilar points. Like one of earliest metric learning methods MMC, we propose to find a matrix $\ma$ that decreases the sum of distances over all the similar points, but unlike all previous methods, instead of treating dissimilar points asymmetrically, we  propose to measure their interpoint distances using $\ma^{-1}$, and to add their contribution to the overall objective. More precisely, we propose the following novel objective function:
\begin{equation}
\label{eq:basic}
\sum_{(\vx_i , \vx_j) \in \mathcal{S}}{d_{\ma} (\vx_i,\vx_j)} \quad +  \sum_{(\vx_i , \vx_j) \in \mathcal{D}}{d_{{\ma}^{-1}} (\vx_i,\vx_j)}.
\end{equation}
In the sequel, we write $\hat{d}_{\ma} \equiv d_{\ma^{-1}}$ for brevity.

\subsection{Insights} 
Let us provide some intuition behind our proposed objective~\eqref{eq:basic}. These insights are motivated by the idea that we \emph{may} increase the Mahalanobis distance between dissimilar points $d_{\ma}(\vx,\vy)$ by decreasing $\hat{d}_{\ma}(\vx,\vy)$. The first idea is the simple observation that the distance $d_{\ma}(\vx,\vy)$ increases monotonically in $\ma$, whereas the distance $\hat{d}_{\ma}(\vx,\vy)$ decreases monotonically in $\ma$. This observation follows from the following well-known result: 
\begin{proposition}
  Let $\ma, \mb$ be (strictly) positive definite matrices such that $\ma \succ \mb$. Then, $\ma^{-1} \prec \mb^{-1}$.
\end{proposition}

The second idea (which essentially reaffirms the first) is that the gradients of $d_{\ma}$ and $\hat{d}_{\ma}$ point in nearly opposite directions. Therefore, infinitesimally decreasing $d_{\ma}$ leads to an increase in $\hat{d}_{\ma}$. Indeed, the (Euclidean) gradient of $d_{\ma}(\vx,\vy)$ is
\begin{equation*}
  \frac{\partial d_{\ma}}{\partial \ma} = \vu\vu^{T},
\end{equation*}
where $\vu = \vx-\vy$; this is a rank-one positive semi-definite matrix. The gradient of $\hat{d}_{\ma}(\vx,\vy)$ is
\begin{equation*}
  \frac{\partial \hat{d}_{\ma}}{\partial \ma} = -\ma^{-1}\vu\vu^{T} \ma^{-1} ,
\end{equation*}
which is a rank-one matrix with a negative eigenvalue. It is easy to see that the inner product of these two gradients is negative, as desired.




\subsection{Optimization problem and its solution}
In the following, we further simplify the objective in~\eqref{eq:basic}. Rewriting the Mahalanobis distance using traces, we turn~\eqref{eq:basic} into the optimization problem
\begin{equation}
  \label{eq:2}
  \begin{split}
    \min_{\ma \succ 0} \quad &\sum_{(\vx_i , \vx_j) \in \mathcal{S}}{\trace(\ma(\vx_i-\vx_j)(\vx_i-\vx_j)^{T})} \\
    &+  \sum_{(\vx_i , \vx_j) \in \mathcal{D}}{\trace(\ma^{-1}(\vx_i-\vx_j)(\vx_i-\vx_j)^{T})}.
  \end{split}
\end{equation}
We define now the following two important matrices:
\begin{equation}
  \label{eq:3}
  \begin{split}
    \ms := \sum_{(\vx_i , \vx_j) \in \mathcal{S}}{(\vx_i-\vx_j)(\vx_i-\vx_j)^{T}}, \\
    \md := \sum_{(\vx_i , \vx_j) \in \mathcal{D}}{(\vx_i-\vx_j)(\vx_i-\vx_j)^{T}},
  \end{split}
\end{equation}
which denote the similarity and dissimilarity matrices, respectively. The matrices $\ms$ and $\md$ are scaled second sample moments of the differences between similar points and the differences between dissimilar points. In the rest of this subsection, we assume that $\ms$ is a SPD matrix, which is a realistic assumption in many situations. For the cases where $\ms$ is just a positive semi-definite matrix, the regularized version can be used; we treat this case in Section~\ref{sec:regul}.

Using~\eqref{eq:3}, the minimization problem~\eqref{eq:2} yields the basic optimization formulation of \algo, namely
\begin{equation}
  \label{firstcost}
  \min_{\ma \succ 0} \quad h(\ma) := \trace(\ma \ms) + \trace(\ma^{-1}\md).
\end{equation}

The \algo cost function~\eqref{firstcost} has several remarkable properties, which may not be apparent at first sight. Below we highlight some of these to help build greater intuition, as well as to help us minimize it.

The first key property of $h(\ma)$ is that it is both strictly convex and strictly geodesically convex. Therefore, if $\nabla h(\ma)=0$ has a solution, that solution will be the global minimizer. Before proving this key property of $h$, let us recall some material that is also helpful for the remainder of the section.

Geodesic convexity is the generalization of ordinary (linear) convexity to (nonlinear) manifolds and metric spaces~\cite{athan,rapsak}. On Riemannian manifolds, geodesics are curves with zero acceleration that at the same time locally minimize the Riemannian distance between two points. The set of SPD matrices forms a Riemannian manifold of nonpositive curvature~\cite[Ch.~6]{bhatia07}.
We denote this manifold by $\mathbb{S}_{+}$. The geodesic curve joining $\ma$ to $\mb$ on the SPD manifold is denoted by
\begin{equation*}
\ma \sharp_t \mb = \ma^{1/2}  \bigl ( \ma^{-1/2} \mb \ma^{-1/2}  \bigr )^t \ma^{1/2}, \qquad t \in [0,1].
\end{equation*}
This notation for geodesic is customary, and in the literature, $\gamma(t)$ is also used. Moreover, the entire set of SPD matrices is geodesically convex, as there is a geodesic between every two points in the set. On this set, one defines geodesically convex functions as follows.
\begin{definition}
A function $f$ on a geodesically convex subset of a Riemannian manifold is \emph{geodesically convex}, if for all points $\ma$ and $\mb$ in this set, it satisfies
\begin{equation*}
f( \ma \sharp_t \mb) \leq tf(\ma) + (1-t) f(\mb),\quad t\in[0,1].
\end{equation*}
If for $t \in (0,1)$ the above inequality is strict, the function is called strictly geodesically convex. 
\end{definition}

We refer the reader to~\cite{sra15} for more on geodesic convexity for SPD matrices. We are ready to state a simple but key convexity result.
\begin{theorem}
  \label{thm:gcvx}
  The cost function $h$ in~\eqref{firstcost} is both strictly convex and strictly geodesically convex on the SPD manifold.
\end{theorem}
\begin{proof}
  The first term in~\eqref{firstcost} is linear, hence convex, while the second term is strictly convex~\cite[Ch.~3]{boyd2004convex}, viewing SPD matrices as a convex cone~\cite[see][Thm.~2.6]{rockafellar1970convex}. Thus, strict convexity of $h(\ma)$ is obvious. Therefore, we concentrate on proving its strict geodesic convexity. Using continuity, it suffices to show \emph{midpoint strict convexity}, namely
\begin{equation*}
  h(\ma \sharp_{1/2} \mb) < \tfrac12 h(\ma) + \tfrac12 h(\mb).
\end{equation*}
It is well-known~\cite[Ch.~4]{bhatia07} that for two distinct SPD matrices, we have the operator inequality
\begin{equation}
\ma \sharp_{1/2} \mb \prec \tfrac12 \ma + \tfrac12 \mb.
\end{equation}
Since $\ms$ is SPD, is immediately follows that
\begin{equation}
  \label{eq:4}
  \trace \bigl ((\ma \sharp_{1/2} \mb) \ms) < \tfrac12\trace (\ma \ms) + \tfrac12\trace (\mb \ms).
\end{equation}
From the definition of $\sharp_t$, a brief manipulation shows that 
\begin{equation*}
  (\ma \sharp_{t} \mb)^{-1} = \ma^{-1} \sharp_t \mb^{-1}.
\end{equation*}
Thus, in particular for the midpoint (with $t=1/2$) we have
\begin{equation}
  \label{eq:5}
  \trace \bigl ((\ma \sharp_{1/2} \mb)^{-1} \md) < \tfrac12\trace (\ma^{-1} \md) + \tfrac12\trace (\mb^{-1} \md).
\end{equation}
Adding~\eqref{eq:4} and \eqref{eq:5}, we obtained the desired result.
\end{proof} \vspace{1cm}

\textbf{Solution via geometric mean.}\; The optimal solution to~\eqref{firstcost} will reveal one more reason why we invoke geodesic convexity. Since the constraint set of \eqref{firstcost} is open and the objective is strictly convex, to find its global minimum, it is enough to find a point where the gradient $\nabla h$ vanishes. Differentiating with respect to $\ma$, this yields
\begin{equation*}
  \nabla h(\ma) = \ms - \ma^{-1}\md \ma^{-1}.
\end{equation*}
Setting this gradient to zero results in the equation
\begin{equation}
  \label{eq:6}
  \nabla h(\ma) = 0\ \Longrightarrow \; \ma \ms \ma=\md.
\end{equation}
Equation~\eqref{eq:6} is a Riccati equation whose unique solution is nothing but the midpoint of the geodesic joining $\ms^{-1}$ to $\md$ (see e.g.,~\citet[1.2.13]{bhatia07}). Indeed, 
\begin{align*}
  \ma &= \ms^{-1} \sharp_{1/2} \; \md = \ms^{-1/2}(\ms^{1/2}\md \ms^{1/2})^{1/2}\ms^{-1/2}. 
\end{align*}
Observe by construction this solution is SPD, therefore, the constraint of optimization is satisfied. 

It is this fact that the solution to \algo is given by the midpoint of the geodesic joining the inverse of the second moment matrix of similar points to the second moment matrix of dissimilar points, which gives \algo its name: the midpoint of this geodesic is known as the \emph{matrix geometric mean} and is a very important object in the study of SPD matrices~\cite[Ch.~6]{bhatia07}.

\subsection{Regularized version}
\label{sec:regul}
We have seen that the solution of our method is the geometric mean between  $\ms^{-1}$ and $\md$. However, in practice the matrix $\ms$ might sometimes be non-invertible or near-singular. To address this concern, we propose to add a regularizing term to the objective function. This regularizer term can also be used to incorporate prior knowledge about the distance function. In particular, we propose to use
\begin{equation}
\label{eq.regcost}
\min_{\ma \succ 0} \quad \lambda D_{\text{sld}}(\ma,\ma_0) + \trace(\ma \ms) + \trace(\ma^{-1}\md),
\end{equation}
where $\ma_0$ is the ``prior'' (SPD matrix) and $D_{\text{sld}}(\ma,\ma_0)$ is the symmetrized LogDet divergence: $D_{\ld}(\ma,\ma_0)+D_{\ld}(\ma_0,\ma)$, which is equal to 
\begin{equation}
  \label{eq:7}
D_{\text{sld}}(\ma,\ma_0) := \trace(\ma \ma_0^{-1}) + \trace(\ma^{-1}\ma_0) - 2d,
\end{equation}
where $d$ is the dimensionality of the data.
Interestingly, using~\eqref{eq:7} and following the argument as above, we see that the minimization problem in~\eqref{eq.regcost} with this regularizer also has a closed form solution. After straightforward computations, we obtain the following solution
\begin{equation}
\label{eq:reg}
\ma_{\text{reg}} = (\ms + \lambda \ma_0^{-1})^{-1} \sharp_{1/2} \; (\md + \lambda \ma_0),
\end{equation}
the regularized geometric mean of suitably modified $\ms$ and $\md$ matrices. Observe that as the regularization parameter $\lambda \geq 0$ increases, $\ma_{\text{reg}}$ becomes more similar to $\ma_0$.

\subsection{Extension to weighted geometric mean}
The geodesic viewpoint is also key to deciding how one may assign different ``weights'' to the matrices $\ms$ and $\md$ when computing the \algo solution. This viewpoint is important because merely scaling the cost in \eqref{firstcost} to change the balance between $\ms$ and $\md$ is not meaningful as it only scales the resulting solution $\ma$ by a constant.

Given the geometric nature of the \algo's solution, we replace the linear cost in~\eqref{firstcost} by a nonlinear one guided by Riemannian geometry of the SPD manifold. The key insight into obtaining a weighted version of \algo comes from a crucial geometric observation. \emph{The minimum of~\eqref{firstcost} is also the minimum to the following optimization problem}:
\begin{equation}
\label{eq:8}
\min_{\ma \succ 0} \quad \delta_{R}^2(\ma,\ms^{-1}) + \delta_{R}^2(\ma,\md),
\end{equation}
where $\delta_{R}$ denotes the Riemannian distance
\begin{equation*}
\delta_{R}(\mx,\my) := \norm{\log(\my^{-1/2}\mx\my^{-1/2})}_{\text{F}} \quad \text{for}\; \mx,\my \succ 0,
\end{equation*}
on SPD matrices and $\norm{\cdot}_{\text{F}}$ denotes the Frobenius norm.

Once we identify the solution of \eqref{firstcost} with that of~\eqref{eq:8}, the generalization to the weighted case becomes transparent. We introduce a parameter that characterizes the degree of balance between the cost terms of similarity and dissimilarity data. The weighted \algo formulation is then
\begin{equation}
  \label{eq:9}
  \min_{\ma \succ 0} \quad h_t(\ma):=(1-t)\; \delta_{R}^2(\ma,\ms^{-1}) + t\; \delta_{R}^2(\ma,\md),
\end{equation}
where $t$ is a parameter that determines the balance. Unlike~\eqref{firstcost}, which we observed to be strictly convex as well as strictly geodesically convex, problem~\eqref{eq:9} is \emph{not} (Euclidean) convex. Fortunately, it is still geodesically convex, because $\delta_R$ itself is geodesically convex. The proof of the geodesic convexity of $\delta_R$ is more involved than that of Theorem~\ref{thm:gcvx}, and we refer the reader to~\cite[Ch.~6]{bhatia07} for complete details.

It can be shown, see e.g.,~\cite[Ch.~6]{bhatia07} that the unique solution to~\eqref{eq:9} is the weighted geometric mean
\begin{equation}
  \ma = \ms^{-1} \sharp_{t} \; \md,
\end{equation}
that is, a point on the geodesic from $\ms^{-1}$ and $\md$. Figure~\ref{fig.manifold} illustrates this fact about the solution of \algo.

\begin{figure}[tbp] 
\centering
\vspace{-1.4cm} 
\newlength\figureheight 
\newlength\figurewidth
\setlength\figureheight{6cm} 
\setlength\figurewidth{6cm} 
\scalebox{.6}{\begin{tikzpicture}
\tikzset{newstyle/.style={thick}}
\tikzset{simple/.style={thick}}

\definecolor{c1}{RGB}{221,52,151}
	\begin{pgfonlayer}{nodelayer}
		\node [style=newstyle] (0) at (-10.75, -3.75) {};
		\node [style=newstyle] (1) at (-4.25, -6.75) {};
		\node [style=newstyle] (2) at (0.25, -5.75) {};
		\node [style=newstyle] (3) at (-6.8, -2.25) {};
		\node [style=newstyle] (4) at (-8.79, -0.13) {};
		\node [style=newstyle] (5) at (-6.5, -0.5) {};
		\node [style=newstyle] (6) at (-3, -4) {};
	\end{pgfonlayer}
	\begin{pgfonlayer}{edgelayer}
		\draw [line width=1.2pt, in=120, out=75, looseness=2.50] (0.center) to (1.center) node[below right] {\huge $\mathbb{S}_{+}$};
		\draw [line width=1.2pt, in=135, out=75] (1.center) to (2.center);
		\draw [line width=1.2pt, in=105, out=9, looseness=1.25] (4.center) to (2.center);
		\draw [line width=1.2pt, in=30, out=180, looseness=1.25] (3.center) to (0.center);
		\draw [color=c1, line width=2pt, in=105, out=15] (5.center) to (6.center);
		\node (c) at (-2.74,-2.85) {\Large $\gamma(t)$};
		\fill (-2.97,-4.1)  circle[radius=4pt];
		\fill (-6.6,-0.52)  circle[radius=4pt];
		\node (c) at (-3.15,-4.6) {\Large $\ms^{-1}$};
		\node (c) at (-6.7,-1) {\Large $\md$};
		\node (c) at (-4.16,-1.45) {\huge $\text{+}$};
		\node at (-4.16,-1.45) [line width=1.4pt, cross out, draw, inner sep=3.2pt]{};
		\node (c) at (-3.65,-1.35) {\Large $\ma$};
	\end{pgfonlayer}
\end{tikzpicture} }
\caption{The solution of \algo is located in the geodesic between matrices $\ms^{-1}$ and $\md$ on the manifold of SPD matrices.\label{fig.manifold}
} 
\end{figure}
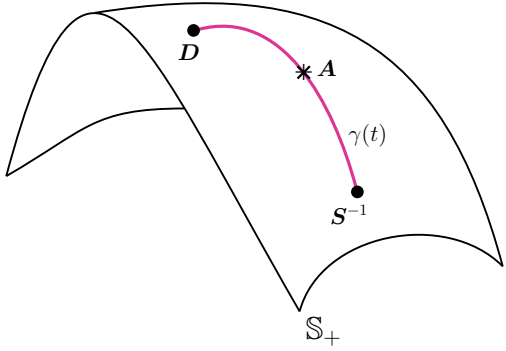




The regularized form of the previous solution is given by
\begin{equation*}
\ma_{\text{reg}} = (\ms + \lambda \ma_0^{-1})^{-1} \sharp_{t} \; (\md + \lambda \ma_0),
\end{equation*}
for $t \in [0,1]$. In the cases where $t = 1/2$, it is equal to~\eqref{eq:reg}. This solution is our final and complete proposed solution to the linear metric learning problem. The summary of our \algo algorithm for metric learning is presented in Algorithm~\ref{alg:GMML}. Empirically, we have observed that the generalized solution (with free $t$) can significantly outperform the ordinary solution.

There are several approaches for fast computation of Riemannian geodesics for SPD matrices, for instance, Cholesky-Schur and scaled Newton methods~\cite{iannazzo2011geometric}. We use Cholesky-Schur method in our paper to expedite the computation of Riemannian geodesics.

\begin{algorithm}[tb]
   \caption{Geometric Mean Metric Learning}
   \label{alg:GMML}
\begin{algorithmic}
   \STATE {\bfseries Input:} $\mathcal{S}$: set of similar pairs, $\mathcal{D}$: set of dissimilar pairs,
   $t$: step length of geodesic, $\lambda$: regularization parameter, $\ma_0$: prior knowledge   
   \STATE {\bfseries Compute the similarity and dissimilarity matrices:}
\begin{align*}
\ms &= \sum_{(\vx_i , \vx_j) \in \mathcal{S}}{(\vx_i-\vx_j)(\vx_i-\vx_j)^{T}} \\
\md &= \sum_{(\vx_i , \vx_j) \in \mathcal{D}}{(\vx_i-\vx_j)(\vx_i-\vx_j)^{T}}
\end{align*}
\vspace{-0.2cm}
\STATE {\bfseries Return the distance matrix:}
\begin{equation*}
\ma = (\ms + \lambda \ma_0^{-1})^{-1} \sharp_{t} \; (\md + \lambda \ma_0)
\end{equation*}
\end{algorithmic}
\end{algorithm}

    


\section{Results}
\label{sec:results}
In this section, we compare the performance of the proposed method \algo (Algorithm~\ref{alg:GMML}) to some well-known metric learning algorithms:
\begin{itemize}
\setlength{\itemsep}{0.5pt}
\item ITML~\cite{davis2007information};
\item LMNN~\cite{weinberger2009distance}; and
\item FlatGeo with batch flat geometry~\cite{meyer2011regression}.
\end{itemize}

We exploit the commonly used criterion for comparing the performance of different methods, that is, the rate of the classification error for a $k$-NN classifier on different datasets. We choose $k=5$, and estimate a full-rank matrix $\ma$ in all methods.

\subsection{Experiment 1}
Assume $c$ to be the number of classes, it is common in practice to generate $40c(c-1)$ number of constraints by randomly choosing $40c(c-1)$ pairs of points in a dataset. In our first experiment, shown in Figure~\ref{fig.small}, we use this number of constraints in our method in addition to ITML and FlatGeo methods. The LMNN method does not have this number of constraints parameter and we used a new version of its toolbox that uses Bayesian optimization for optimizing the model hyper-parameters. We use the default parameters used in ITML and FlatGeo, except we also use a minimum iterations of $10^4$ for the FlatGeo method, because we observed that sometimes FlatGeo stops prematurely leading to a very poor performance. ITML has a regularization parameter that is set by using cross-validation.






\begin{figure*}[htbp] 
\centering 
\setlength\figureheight{6cm} 
\setlength\figurewidth{6cm} 
\scalebox{.75}{\begin{tikzpicture}
\begin{axis}[scale=2.6,
axis lines*=left,
width=12.023cm,
height=2.1in,
bar width = .32cm,
enlargelimits=0.07,
    axis lines*=left,
    width=10cm,
    tick align=inside,
    tick style={draw=none},
    xtick={4,6,8,10,12,14,16,18,20,22},
    xticklabels={%
        \small{Wine} \\ \footnotesize{$d=13$, $c=3$ $n=178$},
        \small{Pima-diabetes} \footnotesize{$d=8$, $c=2$ $n=768$},
        \small{Vehicle} \\ \footnotesize{$d=18$, $c=4$ $n=846$},
        \small{Vowel} \\ \footnotesize{$d=14$, $c=11$ $n=990$},
        \small{German} \\ \footnotesize{$d=24$, $c=2$ $n=1000$},
        \small{Australian} \footnotesize{$d=14$, $c=2$ $n=690$},
        \small{Protein} \\ \footnotesize{$d=20$, $c=6$ $n=116$},
        \small{Iris} \\ \footnotesize{$d=4$, $c=3$ $n=150$},
        \small{Breast-Cancer} \footnotesize{$d=9$, $c=2$ $n=699$},
        \small{Segment} \\ \footnotesize{$d=19$, $c=7$ $n=2310$}},
    xticklabel style = { font=\scriptsize, text width=1.99cm, align=center },
    ybar=0pt,
    ymin = 3.5,
    ymax = 52.5,
ytick={0,  5,  10,  15,   20,   25,  30,  35,   40,   45,   50,   55},
ylabel={Classification Error ($\%$)},
    legend image code/.code={%
                    \draw[#1, draw=none] (0cm,-0.1cm) rectangle (0.4cm,0.17cm);
                },  
                legend style={
                    text depth=0pt,
                    at={(0.825,0.83)},
                    anchor=north west,
                    default spacing:
                    /tikz/column 2/.style={column sep=0pt,font=\bfseries},
                    %
                    /tikz/every odd column/.append style={column sep=0cm},
                },
    ]

\definecolor{c1}{RGB}{96,188,159}
\definecolor{c2}{RGB}{252,141,98}
\definecolor{c3}{RGB}{141,160,203}
\definecolor{c4}{RGB}{231,138,195}
\definecolor{c6}{RGB}{255,217,47}

\addplot[
    fill=c1,
    draw=black,
    point meta=y,
    every node near coord/.style={inner ysep=5pt},
    error bars/.cd,
        y dir=both,
        y explicit
] 
table [y error=error] {
x      y     error
4   3.295   1.130
6   27.504    1.322
8   22.521    1.252
10   43.050    2.783
12   27.590    0.957
14  16.004   0.893
16   33.304    3.213
18   2.988   1.117
20   3.547    0.566
22   2.584   0.3622
}; \addlegendentry{GMML}

\addplot[
    fill=c2,
    draw=black,
    point meta=y,
    every node near coord/.style={inner ysep=5pt},
    error bars/.cd,
        y dir=both,
        y explicit
] 
table [y error=error] {
x      y     error
4   8.665    2.047
6   27.952    1.358
8   23.445    1.358
10  48.971    6.758
12  28.900    1.112
14  33.550    3.100
16   34.202    5.261
18   5.2   0.4
20  5.601    0.810
22  2.976   0.369
}; \addlegendentry{LMNN}

\addplot[
    fill=c3,
    draw=black,
    point meta=y,
    every node near coord/.style={inner ysep=5pt},
    error bars/.cd,
        y dir=both,
        y explicit
] 
table [y error=error] {
x       y      error
4   7.584    2.543
6   28.372    1.078
8   29.929    2.175
10   43.722    2.592
12   29.180    1.465
14   33.289    2.365
16   34.224    4.003
18   3.466   1.4673
20   6.766    0.575
22   4.805   0.910
}; \addlegendentry{ITML}

\addplot[
    fill=c4,
    draw=black,
    point meta=y,
    every node near coord/.style={inner ysep=5pt},
    error bars/.cd,
        y dir=both,
        y explicit
] 
table [y error=error] {
x       y     error
4   3.920    0.327
6   27.734    1.932
8   27.955  1.559
10   49.207    1.389
12   30.810    0.964
14   34.2188    3.044
16   34.741    3.149
18   3.666   1.0062
20   6.327    1.051
22   3.896  0.572
}; \addlegendentry{FlatGeo}

\addplot[
    fill=c6,
    draw=black,
    point meta=y,
    every node near coord/.style={inner ysep=5pt},
    error bars/.cd,
        y dir=both,
        y explicit
] 
table [y error=error] {
x       y    error
4   30.632    2.192
6   27.744    1.212
8   38.153   1.294
10   42.207    0
12   30.857    1.512
14   32.351    1.719
16   36.638    3.981
18   3.8   1.2002
20   7.148    0.668
22   4.329   0
}; \addlegendentry{Euclidean}

\end{axis}
\end{tikzpicture} }
\caption{Classification error rates of $k$-nearest neighbor classifier via different learned metrics for different small datasets. Numbers below each correspond to the dimensionality of feature space in the data ($d$), number of classes ($c$) and number of total data ($n$).\label{fig.small}
} 
\end{figure*}
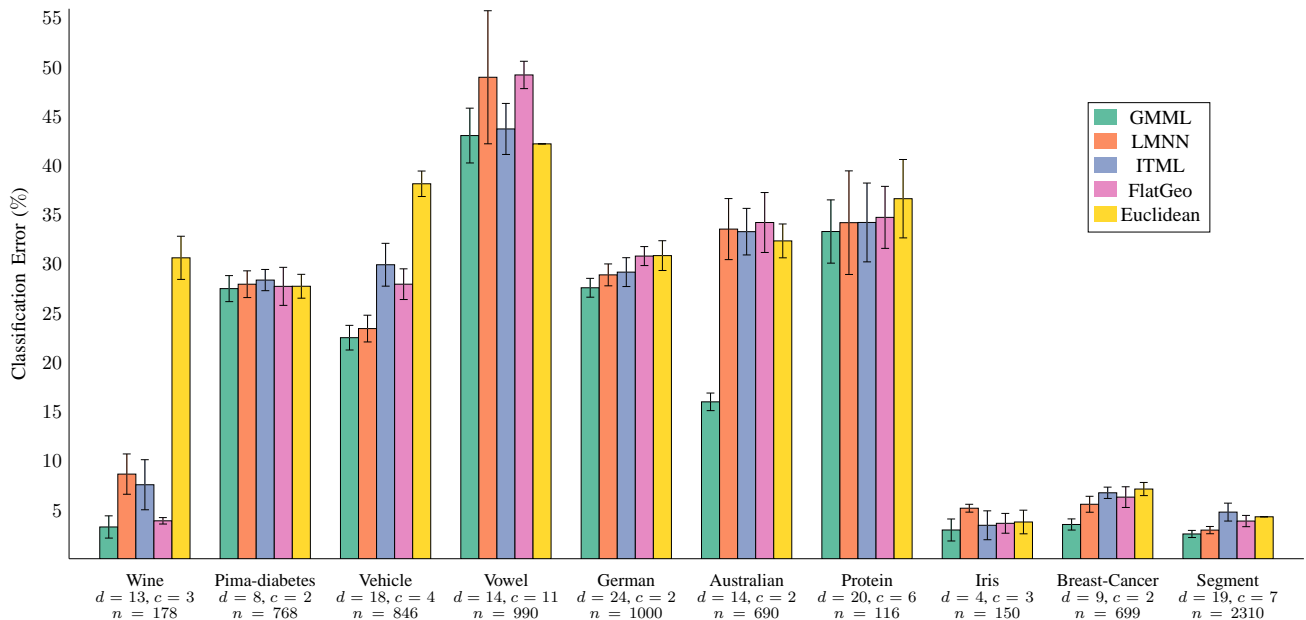

Figure~\ref{fig.small} reports the results for the smaller datasets. The datasets are obtained from the well-known UCI repository~\cite{asuncion2007uci}. In the plot, the baseline of using Euclidean distance for classification is shown in yellow. It can be seen that \algo outperforms the other three metric learning methods. 

The figure reports 40 runs of a two-fold splitting of the data. In each run, the data is randomly divided into two equal sets. The regularization parameter $\lambda$ is set to zero for most of the datasets. We only add a small value of $\lambda$ when the similarity matrix $\ms$ becomes singular. For example, since the similarity matrix of the Segment data is near singular, we use the regularized version of our method with $\lambda=0.1$ and $\ma_0$ equals to the identity matrix.

We use five-fold cross-validation for choosing the best parameter $t$. We tested 18 different values for $t$ in a two-step method. In the first step the best $t$ is chosen among the values $\{0.1, 0.3, 0.5, 0.7, 0.9\}$. Then in the second step, 12 values of $t$ are tested within an interval of length $0.02$ in the window around the previously selected point.

Figure~\ref{fig.Effects} shows the effect of the parameter $t$ on the average accuracy of $k$-NN classifier for five datasets. These datasets are also appeared in Figure~\ref{fig.small}. It is obvious that in some datasets, going from the ordinary version to the extended version can make the \algo's performance substantially better. Observe that each curve has a convex-like shape with some wiggling. That is why we choose the above approach for finding the best $t$, and we can verify its precision by comparing Figures~\ref{fig.small} and~\ref{fig.Effects}.

\begin{figure}[htbp] 
\centering 
\setlength\figureheight{6cm} 
\setlength\figurewidth{6cm} 
\scalebox{.65}{
%
%
\definecolor{mycolor1}{rgb}{1.00000,1.00000,0.00000}%
\begin{tikzpicture}

\begin{axis}[scale=0.7,
width=6.023in,
height=4.75in,
at={(1.01in,0.641in)},
scale only axis,
separate axis lines,
every outer x axis line/.append style={black},
every x tick label/.append style={font=\color{black}},
xmin=0,
xmax=1,
xlabel={$t$},
every outer y axis line/.append style={black},
every y tick label/.append style={font=\color{black}},
ymin=0,
ymax=50,
ytick={0,  5,  10,  15,    20,   25,  30,  35,    40,    45,  50},
ylabel={Classification Error ($\%$)},
axis background/.style={fill=white},
legend style={at={(0.65,0.49)},anchor=north west,legend cell align=left,align=left,draw=black}
]

\definecolor{c1}{RGB}{127,201,127}
\definecolor{c2}{RGB}{190,174,212}
\definecolor{cj4}{RGB}{253,192,134}
\definecolor{c5}{RGB}{56,108,176}
\definecolor{c3}{RGB}{240,2,127}
\definecolor{c4}{RGB}{191,91,23}

\addplot [color=c5,solid,line width=1.5pt]
  table[row sep=crcr]{%
0	2.81573033707919\\
0.02	2.7797752808994\\
0.04	2.9382022471915\\
0.06	3.257303370787\\
0.08	3.57415730337124\\
0.1	3.92247191011282\\
0.12	4.38651685393301\\
0.14	4.89101123595538\\
0.16	5.24831460674188\\
0.18	5.72808988764074\\
0.2	6.73370786516881\\
0.22	8.08089887640475\\
0.24	9.6056179775282\\
0.26	11.2337078651687\\
0.28	13.0808988764045\\
0.3	15.2719101123596\\
0.32	17.5730337078651\\
0.34	19.88202247191\\
0.36	22.1977528089886\\
0.38	24.356179775281\\
0.4	25.9786516853933\\
0.42	27.4314606741575\\
0.44	28.6101123595508\\
0.46	29.6067415730339\\
0.48	30.5393258426969\\
0.5	30.8910112359553\\
0.52	31.3483146067418\\
0.54	31.688764044944\\
0.56	31.9415730337081\\
0.58	32.1898876404496\\
0.6	32.4022471910114\\
0.62	32.6224719101125\\
0.64	32.503370786517\\
0.66	32.6921348314609\\
0.68	32.6685393258429\\
0.7	32.8067415730339\\
0.72	32.9022471910115\\
0.74	32.9988764044946\\
0.76	32.965168539326\\
0.78	32.8786516853934\\
0.8	33.0022471910114\\
0.82	32.9685393258429\\
0.84	33.0157303370788\\
0.86	33.1348314606744\\
0.88	33.0483146067418\\
0.9	33.0224719101125\\
0.92	33.1370786516856\\
0.94	33.1067415730339\\
0.96	32.9235955056182\\
0.98	33.0707865168541\\
1	33.0561797752811\\
};
\addlegendentry{Wine};

\addplot [color=c3,solid,line width=1.5pt]
  table[row sep=crcr]{%
0	22.0947990543736\\
0.02	21.9477541371159\\
0.04	21.9888888888889\\
0.06	21.8732860520095\\
0.08	22.1851063829787\\
0.1	22.5557919621749\\
0.12	22.8120567375886\\
0.14	23.1647754137116\\
0.16	23.5978723404256\\
0.18	23.9737588652483\\
0.2	24.4617021276596\\
0.22	24.976122931442\\
0.24	25.5931442080378\\
0.26	26.1624113475177\\
0.28	26.8761229314421\\
0.3	27.6647754137116\\
0.32	28.4111111111111\\
0.34	29.3638297872341\\
0.36	30.2359338061465\\
0.38	31.2307328605201\\
0.4	32.1827423167849\\
0.42	33.2141843971632\\
0.44	34.2444444444445\\
0.46	35.3215130023641\\
0.48	36.4021276595744\\
0.5	37.2671394799054\\
0.52	38.2858156028369\\
0.54	39.0832151300236\\
0.56	39.8652482269504\\
0.58	40.6170212765957\\
0.6	41.4706855791962\\
0.62	42.0907801418439\\
0.64	42.7198581560283\\
0.66	43.447281323877\\
0.68	44.0004728132388\\
0.7	44.5427895981087\\
0.72	45.1082742316785\\
0.74	45.7167848699763\\
0.76	45.8437352245862\\
0.78	46.4222222222222\\
0.8	46.8548463356973\\
0.82	47.0269503546098\\
0.84	47.3245862884161\\
0.86	47.4326241134751\\
0.88	47.641134751773\\
0.9	47.7139479905437\\
0.92	47.7827423167848\\
0.94	48.0791962174941\\
0.96	48.111111111111\\
0.98	48.2026004728132\\
1	48.269976359338\\
};
\addlegendentry{Vehicle};

\addplot [color=c4,solid,line width=1.5pt]
  table[row sep=crcr]{%
0	29.7313999999999\\
0.02	29.3458\\
0.04	29.004\\
0.06	28.6504\\
0.08	28.4478\\
0.1	28.1974\\
0.12	27.931\\
0.14	27.772\\
0.16	27.5987999999999\\
0.18	27.4813999999999\\
0.2	27.4133999999999\\
0.22	27.2173999999999\\
0.24	27.1973999999999\\
0.26	27.1651999999999\\
0.28	27.3723999999999\\
0.3	27.6049999999999\\
0.32	27.7955999999999\\
0.34	28.2866\\
0.36	28.5831999999999\\
0.38	29.1246\\
0.4	29.4372\\
0.42	29.7424\\
0.44	30.256\\
0.46	30.5414\\
0.48	31.0364\\
0.5	31.266\\
0.52	31.6076\\
0.54	31.8093999999999\\
0.56	32.1902\\
0.58	32.335\\
0.6	32.5948\\
0.62	32.7412000000001\\
0.64	32.8088\\
0.66	33.0526\\
0.68	33.0202\\
0.7	33.1424000000001\\
0.72	33.2712\\
0.74	33.3004000000001\\
0.76	33.3572\\
0.78	33.4394000000001\\
0.8	33.3714\\
0.82	33.3676000000001\\
0.84	33.5176\\
0.86	33.4696\\
0.88	33.5984\\
0.9	33.4823999999999\\
0.92	33.5492\\
0.94	33.5544\\
0.96	33.6668\\
0.98	33.6744\\
1	33.6898\\
};
\addlegendentry{German};

\addplot [color=c1,solid,line width=1.5pt]
  table[row sep=crcr]{%
0	37.0637931034483\\
0.02	36.1189655172415\\
0.04	35.9224137931035\\
0.06	35.2948275862071\\
0.08	35.0534482758623\\
0.1	34.386206896552\\
0.12	34.4741379310347\\
0.14	34.0775862068968\\
0.16	33.6275862068968\\
0.18	33.2517241379313\\
0.2	32.9344827586209\\
0.22	33.1034482758623\\
0.24	32.9689655172416\\
0.26	32.7741379310347\\
0.28	32.6051724137933\\
0.3	32.5103448275865\\
0.32	32.4224137931037\\
0.34	32.03275862069\\
0.36	32.2568965517244\\
0.38	32.1982758620693\\
0.4	32.3241379310347\\
0.42	32.3517241379313\\
0.44	32.4758620689658\\
0.46	32.2655172413796\\
0.48	32.6275862068968\\
0.5	32.9448275862071\\
0.52	33.2310344827589\\
0.54	33.4793103448278\\
0.56	33.4586206896554\\
0.58	33.9965517241381\\
0.6	33.9500000000002\\
0.62	34.3310344827588\\
0.64	34.9655172413794\\
0.66	35.0810344827587\\
0.68	35.1\\
0.7	35.3637931034484\\
0.72	35.9172413793104\\
0.74	36.6051724137932\\
0.76	36.898275862069\\
0.78	36.9206896551725\\
0.8	37.9086206896552\\
0.82	37.8775862068966\\
0.84	38.6448275862069\\
0.86	38.9258620689654\\
0.88	39.4258620689655\\
0.9	39.4672413793103\\
0.92	40.6586206896551\\
0.94	40.7086206896551\\
0.96	40.9637931034482\\
0.98	41.4999999999999\\
1	41.7862068965517\\
};
\addlegendentry{Protein};

\addplot [color=c2,solid,line width=1.5pt]
  table[row sep=crcr]{%
0	4.59971387696721\\
0.02	4.5124463519314\\
0.04	4.44434907010023\\
0.06	4.30615164520763\\
0.08	4.21945636623767\\
0.1	4.15107296137347\\
0.12	4.10987124463534\\
0.14	4.02689556509316\\
0.16	3.98369098712462\\
0.18	3.91244635193151\\
0.2	3.84005722460673\\
0.22	3.79399141630921\\
0.24	3.75250357653797\\
0.26	3.71673819742505\\
0.28	3.67553648068684\\
0.3	3.58111587982844\\
0.32	3.57110157367677\\
0.34	3.52246065808309\\
0.36	3.45779685264674\\
0.38	3.42947067238923\\
0.4	3.3979971387697\\
0.42	3.34506437768238\\
0.44	3.32703862660947\\
0.46	3.32246065808293\\
0.48	3.2844062947067\\
0.5	3.27868383404869\\
0.52	3.24177396280412\\
0.54	3.27439198855508\\
0.56	3.26294706723896\\
0.58	3.2380543633762\\
0.6	3.28612303290416\\
0.62	3.23748211731051\\
0.64	3.26552217453498\\
0.66	3.23261802575107\\
0.68	3.24091559370527\\
0.7	3.25894134477821\\
0.72	3.28526466380544\\
0.74	3.25779685264669\\
0.76	3.27439198855509\\
0.78	3.24320457796849\\
0.8	3.25951359084399\\
0.82	3.2726752503576\\
0.84	3.23319027181682\\
0.86	3.29270386266096\\
0.88	3.30701001430611\\
0.9	3.29241773962798\\
0.92	3.29670958512162\\
0.94	3.27839771101576\\
0.96	3.29098712446347\\
0.98	3.28011444921309\\
1	3.29270386266093\\
};
\addlegendentry{Breast-Cancer};

\end{axis}
\end{tikzpicture}%
 }
\caption{Classification error rates of $k$-nearest neighbor classifier along with \algo for different values of the parameter $t$. We analyze five datasets here, which is also appeared in Figure~\ref{fig.small}.\label{fig.Effects}
} 
\end{figure}
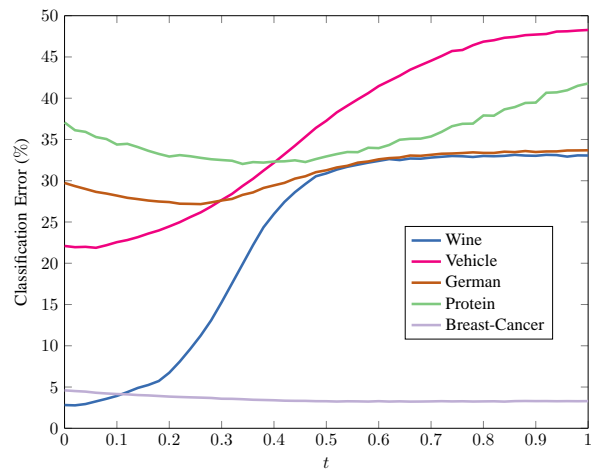

\subsection{Experiment 2}
To evaluate the performance of our method on larger datasets, we conduct a second set of experiments. The results can be summarized in Figure~\ref{fig.big}. The datasets in this experiment are Isolet, Letters~\cite{asuncion2007uci},  MNIST\footnote{We used a smaller version of the MNIST dataset available in www.cad.zju.edu.cn/home/dengcai/Data/MLData.html}~\cite{lecun1998gradient} and USPS~\cite{le1990handwritten}. 

Figure~\ref{fig.big} reports the average classification error over 5 runs of random splitting of the data. We use three-fold cross-validation for adjusting the parameter $t$. Since the similarity matrices of the MNIST data were not invertible, we use the regularized version of our method with regularization parameter $\lambda=0.1$. The prior matrix $\ma_0$ is set to the identity matrix.

On two of the large datasets, Letters and USPS, our method achieves the same performance as the best competing method that is LMNN. For one of the datasets our method significantly outperforms LMNN, and in one dataset it is significantly outdone by LMMN. We also observed that by using more data pairs for generating the similarity and dissimilarity matrices, the performance of our method on Isolet and MNIST datasets improves. We tested $1000c(c-1)$ for these two datasets, with which  we achieve about $1$ percent better accuracy for Isolet leading to slightly better performance than FlatGeo approach. For MNIST data, we achieved about $0.5$ percent better accuracy.

\begin{figure}[t] 
\centering  
\setlength\figureheight{6cm} 
\setlength\figurewidth{6cm} 
\scalebox{.65}{\begin{tikzpicture}
\begin{axis}[scale=1.43,
axis lines*=left,
width=8.023cm,
height=2.1in,
bar width = .38cm,
enlargelimits=0.145,
    axis lines*=left,
    width=10cm,
    tick align=inside,
    tick style={draw=none},
    xtick={6,8,10,12},
    xticklabels={%
        \small{USPS} \\ \footnotesize{$d=256$, $c=10$ $n=9298$},
        \small{MNIST} \\ \footnotesize{$d=784$, $c=10$ $n=4000$},
        \small{Isolet} \\ \footnotesize{$d=617$, $c=26$ $n=7797$},
        \small{Letters} \\ \footnotesize{$d=16$, $c=26$ $n=20000$}},
     xticklabel style = { font=\scriptsize, text width=2.2cm, align=center },
    ybar=0pt,
    ymin = 2.3,
    ymax = 18,
ytick={0,  5,  10,  15,    20},
ylabel={Classification Error ($\%$)},
    legend image code/.code={%
                    \draw[#1, draw=none] (0cm,-0.1cm) rectangle (0.4cm,0.17cm);
                },  
                legend style={
                    text depth=0pt,
                    at={(0.745,0.99)},
                    anchor=north west,
                    default spacing:
                    /tikz/column 2/.style={column sep=0pt,font=\bfseries},
                    %
                    /tikz/every odd column/.append style={column sep=0cm},
                },
    ]

\definecolor{c1}{RGB}{96,188,159}
\definecolor{c2}{RGB}{252,141,98}
\definecolor{c3}{RGB}{141,160,203}
\definecolor{c4}{RGB}{231,138,195}
\definecolor{c6}{RGB}{255,217,47}


\addplot[
    fill=c1,
    draw=black,
    point meta=y,
    every node near coord/.style={inner ysep=5pt},
    error bars/.cd,
        y dir=both,
        y explicit
] 
table [y error=error] {
x      y     error
6   2.942   0.124
8   11.900   0.271
10   6.892    0.643
12   4.092   0.172
}; \addlegendentry{GMML}

\addplot[
    fill=c2,
    draw=black,
    point meta=y,
    every node near coord/.style={inner ysep=5pt},
    error bars/.cd,
        y dir=both,
        y explicit
] 
table [y error=error] {
x      y     error
6   2.552   0.255
8   18.44   1.133
10   3.688    0.3175
12   4.082   0.088
}; \addlegendentry{LMNN}

\addplot[
    fill=c3,
    draw=black,
    point meta=y,
    every node near coord/.style={inner ysep=5pt},
    error bars/.cd,
        y dir=both,
        y explicit
] 
table [y error=error] {
x       y      error
6   3.057   0.393
8   13.85   1.115
10   8.654    0.72  
12   5.585   0.184
}; \addlegendentry{ITML}

\addplot[
    fill=c4,
    draw=black,
    point meta=y,
    every node near coord/.style={inner ysep=5pt},
    error bars/.cd,
        y dir=both,
        y explicit
] 
table [y error=error] {
x       y     error
6   7.272   1.916
8   14.21   1.522
10   5.9    0.092
12   7.280   0.489
}; \addlegendentry{FlatGeo}

\addplot[
    fill=c6,
    draw=black,
    point meta=y,
    every node near coord/.style={inner ysep=5pt},
    error bars/.cd,
        y dir=both,
        y explicit
] 
table [y error=error] {
x       y    error
6   3.585   0
8   14.100   0
10    8.916   0   
12   5.400   0
}; \addlegendentry{Euclidean}

\end{axis}
\end{tikzpicture} }
\caption{Classification error rates of $k$-nearest neighbor classifier via different learned metrics for large datasets.\label{fig.big}
} 
\end{figure}
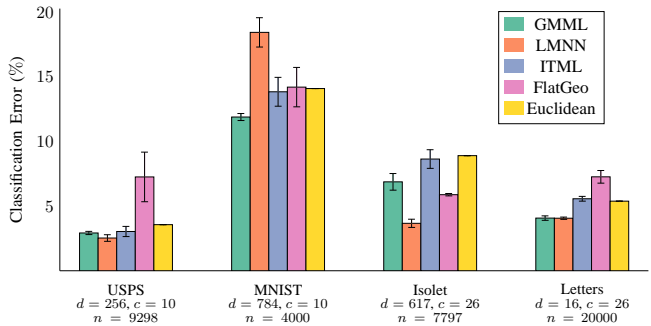



The average running times of the methods on all large data sets and one small dataset are shown in Table~\ref{table.times}. The running time of different methods is reported for only one run of each algorithm for fixed values of hyper-parameters; that means, the reported run times \emph{do not} include the time required to select the hyper-parameters. All methods were implemented on \textsc{Matlab} R2014a (64-bit), and the simulations were run on a personal laptop with an Intel Core i5 (2.5Ghz) processor under the OS X Yosemite operating system.




It can be seen that our method is several order of magnitudes faster than other methods. In addition to obtaining good classification accuracy using the proposed method, the computational complexity of our method is another nice property making it an interesting candidate for large-scale metric learning.


\begin{table}
\scriptsize
\caption{Running time (in seconds) of metric learning methods}
\label{table.times}
\begin{center}
\begin{small}
\begin{sc}
\begin{tabular}{lccccr}
\hline
\vspace{0.01cm}
Data set & GMML & LMNN & ITML & FlatGeo \\
\hline
Segment    & 0.0054 & 77.595 & 0.511 & 63.074 \\
Letters     & 0.0137 & 401.90 & 7.053 & 13543 \\
USPS      & 0.1166 & 811.2 & 16.393 & 17424 \\
Isolet       & 1.4021 & 3331.9 & 1667.5 & 24855 \\
MNIST     & 1.6795 & 1396.4 & 1739.4 & 26640 \\
\hline
\end{tabular}
\end{sc}
\end{small}
\end{center}
\end{table}

\section{Conclusion and future work}
We revisited the task of learning a Euclidean metric from weakly supervised data given as pairs of similar and dissimilar points. Building on geometric intuition, we approached the task of learning a symmetric positive definite matrix by formulating it as a smooth, strictly convex optimization problem (thus, ensuring a unique solution). Remarkably, our formulation was shown to have a closed form solution. We also viewed our formulation as an optimization problem on the Riemannian manifold of SPD matrices, a viewpoint that proved crucial to obtaining a weighted generalization of the basic formulation. We also presented a regularized version of our problem. In all cases, the solution could be obtained as a closed form ``matrix geometric mean'', thus explaining our choice of nomenclature.

We experimented with several datasets, both large and small, in which we compared the classification accuracy of a $k$-NN classifier using metric learned via various competing methods. In addition to good classification accuracy and global optimality, our proposed method for solving the metric learning problem has other nice properties like being fast and being scalable with regard to both the dimensionality $d$ and the number of training samples $n$. 

Given the importance of metric learning to a vast number of applications, we believe that the new understanding offered by our formulation, its great simplicity, and its tremendous speedup over widely used methods make it attractive. 

\subsection{Future work}

Several avenues of future work are worth pursuing. We list some most promising directions below:
\begin{itemize}\vspace*{2pt}
  \setlength{\itemsep}{2pt}
\item To view our metric learning methods as a dimensionality reduction method; here the connections in~\cite{cunningham2015linear} may be helpful.
\item Extensions of our simple geometric framework to learn nonlinear and local metrics.
\item Applying the idea of using concurrently the Mahalanobis distance $d_{\ma}$ with its counterpart $\hat{d}_{\ma}$ on the other machine learning problems.
\end{itemize}

\bibliographystyle{IEEEtranN}
\bibliography{IEEEabrv,SML}

\end{document}